\documentclass[sigconf]{acmart}

\AtBeginDocument{%
  \providecommand\BibTeX{{%
    \normalfont B\kern-0.5em{\scshape i\kern-0.25em b}\kern-0.8em\TeX}}}

\copyrightyear{2023}
\acmYear{2023}
\setcopyright{acmlicensed}
\acmConference[MM '23] {Proceedings of the 31st ACM International Conference on Multimedia}{October 29--November 3, 2023}{Ottawa, ON, Canada.}
\acmBooktitle{Proceedings of the 31st ACM International Conference on Multimedia (MM '23), October 29--November 3, 2023, Ottawa, ON, Canada}
\acmPrice{15.00}
\acmISBN{979-8-4007-0108-5/23/10}
\acmDOI{10.1145/3581783.3611986}
\acmSubmissionID{1308}

\usepackage{algorithm}
\usepackage{algorithmic}

\usepackage{mathrsfs}
\usepackage{amssymb}
\usepackage{bm}
\usepackage{amsmath}
\usepackage{dsfont}
\usepackage{booktabs}
\usepackage{epstopdf}
\usepackage{multirow}
\usepackage{comment}
\usepackage{endnotes}
\usepackage{hyperref}
\usepackage{bibentry}
\usepackage{ulem}
\usepackage{graphicx}
\usepackage{enumerate}
\usepackage{subfig}

\newtheorem{proposition}{Proposition}

\newtheorem{remark}{Remark}

\begin{document}
\title{Efficient Multi-View Graph Clustering with Local and Global Structure Preservation}

\author{Yi Wen}
\authornote{Both authors contributed equally to this research.}
\email{wenyi21@nudt.edu.com}
\author{Suyuan Liu}
\authornotemark[1]
\email{suyuanliu@nudt.edu.cn}
\affiliation{%
  \institution{National University of Defense Technology}
  \city{Changsha}
  \country{China}
}

\author{Xinhang Wan}
\email{wanxinhang@nudt.edu.cn}
\affiliation{%
  \institution{National University of Defense
Technology}
  \city{Changsha}
  \country{China}
}

\author{Siwei Wang}
\email{wangsiwei13@nudt.edu.com}
\affiliation{%
  \institution{Intelligent Game and Decision Lab}
  \city{Beijing}
  \country{China}
}

\author{Ke Liang}
\email{liangke200694@126.com}
\affiliation{%
  \institution{National University of Defense
Technology}
  \city{Changsha}
  \country{China}
}

\author{Xinwang Liu}
\authornote{Corresponding author}
\email{xinwangliu@nudt.edu.cn}
\affiliation{%
  \institution{National University of Defense
Technology}
  \city{Changsha}
  \country{China}
}

\author{Xihong Yang}
\email{xihong_edu@163.com}
\author{Pei Zhang}
\email{zhangpei@nudt.edu.cn}
\affiliation{%
  \institution{National University of Defense Technology}
  \city{Changsha}
  \country{China}
}
\renewcommand{\shortauthors}{Yi Wen et al.}

\begin{abstract}
% Anchor-based multi-view graph clustering (AMVGC) has received abundant attention owing to its high efficiency and the capability to capture complementary structural information across multiple views. Intuitively, a high-quality anchor graph plays an essential role in the success of AMVGC. However, the existing AMVGC methods only consider single-structure information, i.e., local or global structure, which provides insufficient information for the learning task. To be specific, the over-scattered global structure leads to learned anchors failing to depict the cluster partition well, while the local structure depends on an improper similarity measure resulting in potentially inaccurate anchor assignment, which would ultimately lead to sub-optimal clustering performance. To tackle these issues, we propose a novel anchor-based multi-view graph clustering framework termed Efficient Multi-View Graph Clustering with Local and Global Structure Preservation (EMVGC-LG). Specifically, EMVGC-LG simultaneously considers the two structures to refine clustering performance. Besides, anchor construction and graph learning are jointly optimized in our unified framework to enhance the clustering quality. Moreover, we theoretically prove that the proposed paradigm with the global structure can well approximate the local information. In addition, EMVGC-LG inherits the linear complexity of existing AMVGC methods respecting the sample number, which is time-economical and scales well with the data size. Extensive experiments demonstrate the effectiveness and efficiency of our proposed method.

Anchor-based multi-view graph clustering (AMVGC) has received abundant attention owing to its high efficiency and the capability to capture complementary structural information across multiple views. Intuitively, a high-quality anchor graph plays an essential role in the success of AMVGC. However, the existing AMVGC methods only consider single-structure information, i.e., local or global structure, which provides insufficient information for the learning task. To be specific, the over-scattered global structure leads to learned anchors failing to depict the cluster partition well. In contrast, the local structure with an improper similarity measure results in potentially inaccurate anchor assignment, ultimately leading to sub-optimal clustering performance. To tackle the issue, we propose a novel anchor-based multi-view graph clustering framework termed Efficient Multi-View Graph Clustering with Local and Global Structure Preservation (EMVGC-LG). Specifically, a unified framework with a theoretical guarantee is designed to capture local and global information. Besides,  EMVGC-LG jointly optimizes anchor construction and graph learning to enhance the clustering quality. In addition, EMVGC-LG inherits the linear complexity of existing AMVGC methods respecting the sample number, which is time-economical and scales well with the data size. Extensive experiments demonstrate the effectiveness and efficiency of our proposed method.

\end{abstract}



% \begin{CCSXML}
%     <ccs2012>
%     <concept>
%     <concept_id>10010147.10010178.10010224.10010245.10010250</concept_id>
%     <concept_desc>Computing methodologies~Object detection</concept_desc>
%     <concept_significance>500</concept_significance>
%     </concept>
%     </ccs2012>
% \end{CCSXML}

% \ccsdesc[500]{Computing methodologies~Cluster analysis on graphs}

% \ccsdesc[500]{Computing methodologies~Semi-Supervised node classification}


\maketitle

\section{INTRODUCTION}
With the advent of the information society, data are usually extracted from multiple sensors in real-world applications \cite{xu2022multi,li2018survey,wang2021survey, AKGR,liang2023structure, LiangTNNLS, ZJPACMMM}. Take the video for example, a clip could be decomposed into the picture, sound records, and text descriptions, which are obtained from a camera, microphone, and producer, respectively \cite{wan2022continual,liu2023self,wan2023autoweighted}. Regarding each modality as a view, how to effectively integrate information from different views has turned out to be an essential task in multi-view scenarios \cite{wan2023onestep,wan2023fast,xu2023adaptive}. Multi-view graph clustering (MVGC), as the classical unsupervised multi-view methods \cite{huang2019auto,wang2015robust,zhang2018generalized, TGC_ML,liuyue_Dink_net,liuyue_HSAN,liuyue_survey}, which can uncover the intrinsic structure of data pairs, is widely adopted in data mining and knowledge discovery \cite{chen2022neuroadaptive,nie2011spectral,nie2020subspace}. In general, MVGC methods achieve desirable achievement with two essential processes \cite{zhang2020adaptive}, \textit{i.e.,} view-specific graph construction, and the consensus graph fusion. 
For instance,  Li et al. \cite{li2019flexible} construct the fused latent representations with structural information fusion from different views in a weighted manner. % Kang et al. \cite{kang2020partition} integrate graph learning for each view, basic partition generation, and consensus partition fusion in a unified framework to achieve an overall optimal solution

Despite abundant MVGC algorithms proposed in recent years, most of them suffer the cubic computational overhead $\mathcal{O}(n^3)$ and quadratic space overhead $\mathcal{O}(n^2)$ due to the full graph construction and decomposition \cite{wen2020generalized,kang2021structured,wen2021structural}, which hinder their application on large-scale scenarios. As a result,  anchor-based multi-view graph clustering (AMVGC), which selects typical anchors to denote the whole data, is proposed to enhance the algorithm's efficiency. For instance, Li et al. \cite{li2015large} propose an innovative anchor-based multi-view clustering method by fusing the local structure and diverse features. Shi et al. \cite{shi2021fast} generate the indicator matrix by an integrated framework that unifies the anchor graph learning and structure rotation.

% I For instance, Zhan et al. \cite{zhan2018multiview} optimize the final consensus graph by imposing low-rank constraints and minimizing the discrepancies of individual graphs. 
% Zhang et al. \cite{zhang2017latent} reconstruct the data points in the latent space to make the subspace representation more accurate and robust. Liu et al. \cite{Liu2016Multiple} learn the optimal combination of coefficients to generate a unified kernel with less redundant information by imposing matrix-induced regularization terms.
Two basic structures, i.e., global and local structures, both play crucial roles in numerous AMVGC studies. As shown in Figure \ref{explanation} (right), the global structure, which captures the relationship between samples and all anchors, usually has a dense correspondence matrix. Meanwhile, the matrix with local structure is always blocky (with an appropriate permutation) owing to only the relationship between samples and one anchor with high similarity being considered. However, the over-scattered global structure leads to learned anchors failing to depict the cluster partition well, while the local structure with an improper similarity measure results in potentially inaccurate anchor assignment. A natural question is whether combining the two structures leads to a better method \cite{AGLI_ML_CIKM}. The effectiveness of similar thought is demonstrated by other studies in vision \cite{wu2021cvt} and language \cite{wu2020lite,peng2022branchformer} tasks. Nonetheless, it is non-trivial to incorporate global and local structures in the MVC field due to the objective equation inconsistency. One pioneer work, \cite{he2003locality} merged global and local structures to learn the similarity of the original data on the kernel space in a single view case. However, their high time expenditure and single-view scenario limit the scalability of the method. To the best of our knowledge, no generalized framework with global and local structure preservation has been proposed in the field of AMVGC. Besides, existing anchor selection usually uses a heuristic sampling strategy and is separated from the graph learning phase, resulting in the clustering performance dependent on the quality of anchor initialization. For instance, Kang et al. \cite{kang2020large} generate fixed anchors in each view by k-means and average the generated anchor graphs into the fusion graph. Because of the randomness and inflexibility of the k-means and sampling strategy, their clustering performance usually exhibits poor stability.
%Secondly, Global structure and local structure learning are essential for generating satisfactory clustering performance by exploiting the intrinsic correlation of data. However, most existing approaches consider only monotonic structural information, which may not provide comprehensive information to distinguish different objects from similar ones, eventually leading to suboptimal clustering performance. Currently, several researches have demonstrated that perserving both global and local structures can effectively enhance the robustness and superiority of the model.
\begin{figure}
    \centering
    \includegraphics[width=0.99\linewidth]{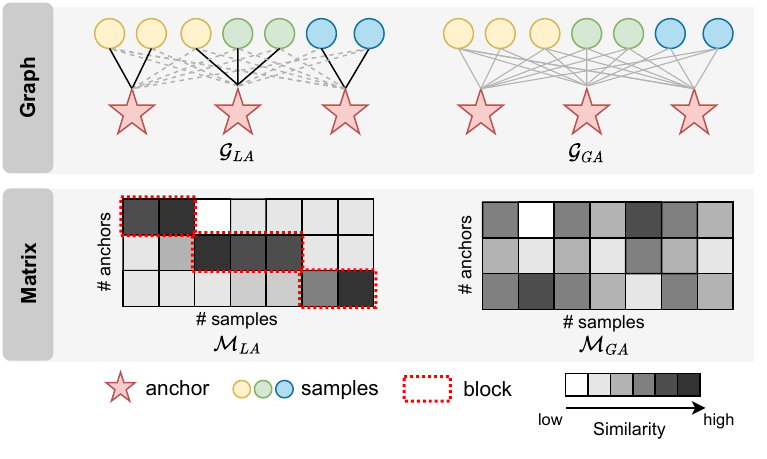}
    \caption{Two types of the anchor graph: local (left) and global (right).  The two types of information often have strong patterns: \textbf{blocky} and \textbf{dense}. $\mathcal{G}_{LA}$ and $\mathcal{M}_{LA}$ denote the visualization and matrix of local anchor graph. $\mathcal{G}_{GA}$ and $\mathcal{M}_{GA}$ denote the visualization and matrix of global anchor graph.}
    \label{explanation}
\end{figure}

To tackle these problems, we design a novel anchor-based multi-view graph clustering framework termed Efficient Multi-View Graph Clustering with Local and Global Structure Preservation (EMVGC-LG). Specifically, a unified framework with a theoretical guarantee is designed to capture local and global information. Besides,  EMVGC-LG jointly optimizes anchor construction and graph learning to enhance the clustering quality. Moreover, we theoretically prove that the proposed paradigm with a global structure can well approximate the local information.
 In addition, EMVGC-LG inherits the linear complexity of existing AMVGC methods respecting the sample number, which is time-economical and scales well with the data size. Meanwhile, a two-step iterative and convergent optimization algorithm is designed in this paper. We summarize the contributions of this paper as follows:
\begin{itemize}
    \item We design an anchor graph learning framework termed Efficient Multi-View Graph Clustering with Local and Global Structure Preservation (EMVGC-LG). With the proven properties, the proposed anchor graph paradigm can not only capture the global structure between data but also well approximate the local structure.

    \item 
     In contrast to existing sampling or fixed anchors, the anchor learning and graph fusion processes are jointly optimized in our framework to enhance the clustering quality.

    \item Extensive experiments on ten benchmark datasets demonstrate the effectiveness and efficiency of our proposed method. 
\end{itemize}

% \begin{table}[t]
% \caption{Main notations}
% \begin{tabular}{lll}
% \toprule
% Notation           & Dimension      & Explanation                                      \\ \hline
% n, v           & -              & Numbers of samples, views \\ 
% k, m            & -              & Numbers of  clusters, anchors \\ 
% $d_p$           & -              & Feature dimension of the p-th view \\ 
% $\mathbf{X}^{(p)}$ & $d_p \times n$ & Data matrix for the p-th view                \\ 
% $\mathbf{S}^{(p)}$ & $n \times n$   & Full graph in the p-th view                  \\ 
% $\mathbf{A}^{(p)}$ & $d_p \times m$ & Specific anchor matrix in the p-th view               \\ 
% $\mathbf{Z}$       & $m \times n$   & Consistent anchor graph                      \\ 
% $\mathbf{Z}^{(p)}$ & $m \times n$   & Specific anchor graph in the p-th view       \\ 
% $\mathbf{F}$       & $n \times k$   & Consistent spectral embedding                  \\ \bottomrule
% \end{tabular}
% \label{note}
% \end{table}

\section{Related Work}
In this section, we present recent research regarding our work, comprising multi-view graph clustering (MVGC) and anchor-based multi-view graph clustering (AMVGC). 

% Table \ref{note} records the main notations throughout this paper.
\subsection{Multi-View Graph Clustering}
With the given dataset $\{\mathbf{X}^{(p)}\}_{p=1}^v \in \mathbb{R}^{d_p \times n}$ consisting of  $n$ samples from $v$ views,  the representative multi-view graph clustering (MVGC) model can express as follows:
$$
\begin{aligned}
& \min_{\mathbf{S}^{(p)},\mathbf{S}} \sum_{p=1}^v \ \left\|\mathbf{X}^{(p)}-\mathbf{X}^{(p)}\mathbf{S}^{(p)}\right\|_F^2+\mu  \mathcal{L}\left(\mathbf{S}^{(p)}, \mathbf{S}\right), \\
& \text { s.t. }\left\{\begin{array}{l}
\operatorname{diag}\left(\mathbf{S}^{(p)}\right)=\mathbf{0}, \mathbf{S}^{(p)\top} \mathbf{1}_n=\mathbf{1}_n, \mathbf{S}^{(p)} \geq \mathbf{0}, \\
\operatorname{diag}(\mathbf{S})=\mathbf{0}, \mathbf{S}^{\top} \mathbf{1}_n=\mathbf{1}_n, \mathbf{S} \geq \mathbf{0},
\end{array}\right.
\end{aligned}
$$
where the first term denotes the data self-representation matrix learning module, and the second term represents the structure integration process performed in $\{\mathbf{S}^{(p)}\}_{p=1}^v$ to generate a common $\mathbf{S}, \mu$ represents a trade-off parameter, $\mathcal{L}$ is the regularization term.  After obtaining the fused global graph  $\mathbf{S}$, we need to obtain the spectral embedding $\mathbf{F} \in \mathbb{R}^{n \times k}$:
\begin{equation}
\min_{\mathbf{F}}\operatorname{Tr}\left(\mathbf{F}^{\top} \mathbf{L} \mathbf{F}\right) ,s.t. \mathbf{F}^{\top}\mathbf{F}=\mathbf{I}_{k},
\end{equation}
where $\mathbf{L}$ denotes the graph Laplace operator, which is calculated by $\mathbf{D}-\mathbf{W}$.  $\mathbf{D}$ is a diagonal matrix, and its element is defined as $d_{ii}=\sum_{j=1}^{n} s_{ij}$. $\mathbf{W}$ is the symmetric similarity matrix, which is calculated by $\frac{\mathbf{S}+\mathbf{S}^\top}{2}$. $\mathbf{F}$ is the spectral representation \cite{ng2002spectral,hamilton2020graph}, and the final clustering label is acquired after the $k$-mean algorithm.

Numerous methods have been proposed on the basis of this framework by imposing different constraints \cite{nie2016constrained,yin2018subspace,li2021consensus,jin2023deep, LiangTKDE} or exploring various types of regularization terms \cite{yin2018subspace,yang2019subspace,zhang2017latent}.
Nie et al. \cite{nie2016parameter} generate the optimal Laplacian matrices by the linear combinations of Laplacian basis matrices constructed from multi-view samples. Gao et al. \cite{gao2015multi}  learn independent representations from each view to capture the diverse information and use a consensus clustering structure to ensure intra-view consistency. Zhang et al.  \cite{zhang2015low} take the view-specific subspace representation as a tensor and explore the intersection information from the diverse views by using low-rank constraints.

Nevertheless, such a paradigm cannot prevent the whole graph construction and consequent spectral decomposition process \cite{kang2019robust,wen2019unified}. The time overhead of the methods is at least $\mathcal{O}\left(n^3\right)$ and the space complexity is at least $\mathcal{O}\left(n^2\right)$, largely hindering the large-scale applications.

\subsection{Anchor-based Multi-View Graph Clustering}

In recent years, anchor-based multi-view graph clustering (AMVGC) has received abundant attention owing to its high efficiency. By constructing the relationship between the representative anchors selected and samples, i.e., anchor graph $\mathbf{Z}^{(p)}$,  to recover the full graph, the space complexity of AMVGC can efficiently reduce from $\mathcal{O}(n^3)$ to  $\mathcal{O}(nm)$ \cite{guo2019anchors,ou2020anchor}.

To our knowledge, a data point within a subspace could be calculated as a linear combination of other data from the same subspace, which is known as the self-expression proposition \cite{ren2020simultaneous,kang2019ijcai}. Therefore, AMVGC methods can count on the self-expression property for the purpose of constructing a reliable anchor graph, which we also refer to as the global structure since it utilizes the whole data information. The classical anchor-based multi-view graph clustering with global structure can be shown as follows:
\begin{equation}
\label{global}
\begin{aligned}
& \min_{\mathbf{Z}}  \ \sum_{p=1}^v \left\|\mathbf{X}^{(p)}-\mathbf{A}^{(p)} \mathbf{Z}\right\|_F^2+\mu \left\| \mathbf{Z}\right\|_F^2, \\ & \text { s.t. } \mathbf{Z}^{\top} \mathbf{1}_m=\mathbf{1}_n, \mathbf{Z} \geq \mathbf{0},
\end{aligned}    
\end{equation}
where $\mathbf{A}^{(p)}$ denotes the anchor matrix from the $p$-th view, and $\mathbf{Z}$ denotes the common anchor graph. The final clustering indicates matrix can be calculated from the SVD decomposition of $\mathbf{Z}$  \cite{kang2020large,xie2020joint,xie2019multiview}. Consequently, the computational and space expenditure is reduced from $\mathcal{O}(n^3)$ to $\mathcal{O}(v n m)$, where $n, m$, and $v$ denotes the number of samples, anchors, and views, correspondingly.

Based on a specific similarity measure, local structure only considers the relationship between samples and one anchor with high similarity. The significance of preserving local manifold structure has been well recognized in non-linear model and cluster analysis \cite{ren2020simultaneous,wen2021global} since neighboring samples usually maintain consistent label information. A classical anchor-based multi-view graph clustering with local structure can be calculated by 
\begin{equation}\label{local}
\begin{split}
     & \min_{\mathbf{z}} \sum_{p=1}^v \sum_{i=1}^n\sum_{j=1}^m\left\|\mathbf{x}_i^{(p)}-\mathbf{a}_j^{(p)}\right\|_2^2 \mathbf{z}_{j i}+\mu_1 \sum_{i=1}^n\sum_{j=1}^m \mathbf{z}_{j i}^2 \ ,\\
     &\text { s.t. } \mathbf{z}_i \geq 0, \mathbf{z}_i^{\top} \mathbf{1}=\mathbf{1}\text{ ,}
  \end{split}
\end{equation}
where $\mathbf{x}_i^{(p)}$ is the $i$-th data from the $p$-th view, $\mathbf{a}_j^{(p)}$ is the $j$-th anchor from the $p$-th view, $\mathbf{z}_i$ denotes the $i$th column of $\mathbf{Z}$, and $\mu$ is the balance parameter that constrains $\mathbf{z}_i \geq 0$ and $\mathbf{z}_i^{\top} \mathbf{1}=\mathbf{1}$ guarantee the probabilistic properties of $\mathbf{z}_i$.

\section{Method}
\subsection{Problem Formulation}
Intuitively, a high-quality graph plays an important role in the success of graph-based clustering \cite{chen2020multi,huang2019ultra,li2022high, CCGC, GCC-LDA,liuyue_SCGC,liuyue_DCRN, MGCN}. Most multi-view graph clustering(MVGC) methods usually adopt a self-representation strategy to characterize the samples. Although the global structure is well-explored, the local structure is ignored, which provides insufficient information for the learning task and ultimately leads to sub-optimal clustering performance \cite{huang2022fast,wang2015multi,gao2020tensor}. Several attempts have been made to address the issues \cite{ren2012local}. For instance, He et al. \cite{he2003locality} merged global and local structures to generate the similarity of the original data in the kernel space. Wen et al. \cite{wen2021global} uses low-rank constraints to produce adaptive graphs for the purpose of one-step clustering. However, their high time and space expenditure hinder the application of the method. 

In this paper, we employ an anchor strategy, which selects representative samples to capture the manifold structure. Besides, we adopt the view-independent anchor and generate a common anchor graph to efficiently excavate the complementary and consistent information of multiple views. With the view-specific anchor  $\mathbf{A}^{(p)} \in \mathbb{R}^{d_p \times m}$ and the consistent anchor graph $\mathbf{Z} \in \mathbb{R}^{m \times n}$, the classical anchor-based multi-view graph clustering with global structure Eq.\eqref{global} can be formulated into the following equivalence problem:
\begin{equation}
\begin{aligned}
   & \min_{\mathbf{Z}} \sum_{p=1}^v \operatorname{tr}\left(\mathbf{Z}^{\top} \mathbf{A}^{(p)\top} \mathbf{A}^{(p)} \mathbf{Z}\right) -2 \sum_{p=1}^v \operatorname{tr}\left(\mathbf{X}^{(p)\top} \mathbf{A}^{(p)} \mathbf{Z}\right) +\mu \operatorname{tr}\left(\mathbf{Z}^{\top} \mathbf{Z}\right), \\
& s.t. \quad
	 \mathbf{Z} \geq 0, \mathbf{Z}^{\top} \mathbf{1}=\mathbf{1},
\end{aligned}
\end{equation}
where the $\mu$ is balanced hyperparameter of regularization term.

For local structure preserving, we summarize the paradigm of the traditional AMVGC with local structure and introduce the terms $\operatorname{tr}(\mathbf{A}^{(p)} \operatorname{diag}(\mathbf{Z} \mathbf{1}) \mathbf{A}^{(p)\top})$, which can be mathematically derived from numerous methods, including BIMVC\cite{li2020bipartite}, MVASM\cite{han2020multi}, and MGLSMC\cite{jiang2021multiple}. With the local term, our objective equation becomes:

\begin{equation}
\begin{split}
&\min_{\mathbf{Z}}  \sum_{p=1}^v\operatorname{tr}\left(\mathbf{Z}^{\top} \mathbf{A}^{(p)\top} \mathbf{A}^{(p)} \mathbf{Z}\right) -2 \sum_{p=1}^v\operatorname{tr}\left(\mathbf{X}^{(p)\top} \mathbf{A}^{(p)} \mathbf{Z}\right) \\
& \quad \quad   +\lambda \sum_{p=1}^v\operatorname{tr}\left(\mathbf{A}^{(p)} \operatorname{diag}(\mathbf{Z} \mathbf{1}) \mathbf{A}^{(p)\top}\right)+ \mu \operatorname{tr}\left(\mathbf{Z}^{\top} \mathbf{Z}\right)
\text {, }\\
&s.t. \quad
	 \mathbf{Z} \geq 0, \mathbf{Z}^{\top} \mathbf{1}=\mathbf{1},
  \end{split}
\end{equation}
where $\mathbf{X}^{(p)} \in \mathbb{R}^{d_p \times n}$ represents the orignial data with $d_p$ dimensions in the $p$-th view.

Although the local and global structures are preserved, the quality of our model is highly dependent on the anchor initiation. Most existing anchor-based multi-view graph clustering algorithms use heuristic sampling strategies such as K-means to acquire desirable anchors. However, the anchor construction, as well as the structure learning are separated from each other, which could restrict the clustering capability. Unlike traditional techniques, we learn anchors automatically instead of sampling in this paper. Finally, we can define the optimization for EMVGC-LG as follows:
\begin{equation}
\begin{split}
\label{my formula}
&\min_{\mathbf{A}^{(p)} , \mathbf{Z}}  \sum_{p=1}^v\operatorname{tr}\left(\mathbf{Z}^{\top} \mathbf{A}^{(p)\top} \mathbf{A}^{(p)} \mathbf{Z}\right)  -2 \sum_{p=1}^v\operatorname{tr}\left(\mathbf{X}^{(p)\top} \mathbf{A}^{(p)} \mathbf{Z}\right)\\
& \quad \quad  +\lambda \sum_{p=1}^v\operatorname{tr}\left(\mathbf{A}^{(p)} \operatorname{diag}(\mathbf{Z} \mathbf{1}) \mathbf{A}^{(p)\top}\right) + \mu \operatorname{tr}\left(\mathbf{Z}^{\top} \mathbf{Z}\right)
\text {. }\\
&s.t. \quad
	 \mathbf{Z} \geq 0, \mathbf{Z}^{\top} \mathbf{1}=\mathbf{1}
  \end{split}
\end{equation}

\begin{proposition}
\label{pro}
By setting $\lambda \in \left(0,1\right], \mu = \lambda \mu_1$, minimizing Eq. (\ref{local}) can be
approximated by minimizing Eq. (\ref{my formula}).
\end{proposition}
\begin{proof}
By adding the item  $\sum_{p=1}^v  \operatorname{tr}( \mathbf{X}^{(p)\top} \mathbf{X}^{(p)})$ not related to the optimized variables, Eq. (\ref{my formula}) can be transformed into the following equivalence problem:
$$
\begin{aligned}
 & Eq. (\ref{my formula})  \Leftrightarrow \\
 & \lambda \sum_{p=1}^v \left( \operatorname{tr} \left( \mathbf{X}^{(p)\top} \mathbf{X}^{(p)}\right) -2 \operatorname{tr}\left(\mathbf{X}^{(p)\top} \mathbf{A}^{(p)} \mathbf{Z}\right) + \operatorname{tr}(\mathbf{M}^{(p)} \mathbf{Z}) \right)
\\
& + \left(1- \lambda\right) \sum_{p=1}^v\left(\operatorname{tr}\left(\mathbf{Z}^{\top} \mathbf{A}^{(p)\top} \mathbf{A}^{(p)} \mathbf{Z}\right)  - 2 \operatorname{tr}\left(\mathbf{X}^{(p)\top} \mathbf{A}^{(p)} \mathbf{Z}\right)  \right.
\\ 
&  \left.  + \operatorname{tr}\left(\mathbf{X}^{(p)\top} \mathbf{X}^{(p)}\right) \right) +  \lambda \sum_{p=1}^v\operatorname{tr}\left(\mathbf{Z}^{\top} \mathbf{A}^{(p)\top} \mathbf{A}^{(p)} \mathbf{Z}\right)  +\mu \operatorname{tr}\left(\mathbf{Z}^{\top} \mathbf{Z}\right) 
\\
& =  \lambda \sum_{p=1}^v\operatorname{tr}\left(\mathbf{Z}^{\top} \mathbf{A}^{(p)\top} \mathbf{A}^{(p)} \mathbf{Z}\right) +  \left(1-\lambda\right) \sum_{p=1}^v \left\|\mathbf{X}^{(p)}-\mathbf{A}^{(p)} \mathbf{Z}\right\|_F^2 
\\
&+ \lambda \sum_{p=1}^v \left( \operatorname{tr}\left(\mathbf{X}^{(p)\top} \mathbf{X}^{(p)}\right) -2 \operatorname{tr}\left(\mathbf{X}^{(p)\top} \mathbf{A}^{(p)} \mathbf{Z}\right) + \operatorname{tr}(\mathbf{M}^{(p)} \mathbf{Z}) \right) \\
& +\lambda \mu_1 \operatorname{tr}\left(\mathbf{Z}^{\top} \mathbf{Z}\right) \\
     & \geq \quad \lambda \sum_{i=1}^n\sum_{j=1}^m \left(\sum_{p=1}^v \left\|\mathbf{x}_i^{(p)}-\mathbf{a}_j^{(p)}\right\|_2^2 \mathbf{z}_{j i}+\mu_1 \mathbf{z}_{j i}^2 \right) \Leftrightarrow Eq. (\ref{local}) \\
\end{aligned}
$$
where $\operatorname{tr}(\mathbf{M}^{(p)} \mathbf{Z}) = \operatorname{tr}\left(\mathbf{A}^{(p)} \operatorname{diag}(\mathbf{Z} \mathbf{1}) \mathbf{A}^{(p)\top}\right)$, $\mathbf{M}^{(p)}_{i,:}=\left[\mathbf{q}_1^{(p)} , \cdots , \right.$ $\left.\mathbf{q}_m^{(p)}\right]$, $\mathbf{q}_j^{(p)} = \mathbf{A}^{(p)\top}_{:,j} \mathbf{A}^{(p)}_{:,j}.$
\end{proof}
\begin{remark}
Proposition \ref{pro} illustrates that the proposed Eq.\eqref{my formula} is an upper bound of Eq.\eqref{local}. Therefore,  our model can approximate the local structure by minimizing Eq.\eqref{my formula}. The tradeoff
between the global and local structure preservation is adjusted by the parameter $\lambda$.
\end{remark}
\subsection{Optimization}
When all the variables are considered together, the optimization of  Eq.(\ref{my formula})
becomes a nonconvex problem, and we address it with an alternating optimization algorithm. 

\subsubsection{Optimization of Anchor Matrices \texorpdfstring{$\{\mathbf{A}^{(p)}\}_{p=1}^v$}{}}
When $\{\mathbf{Z}^{(p)}\}_{p=1}^v$ is fixed, the optimization for $\{\mathbf{A}^{(p)}\}_{p=1}^v$ can be written as follows:
\begin{equation}\label{opt a}
\begin{aligned}
 \min_{\mathbf{A}^{(p)}} 
  & \sum_{p=1}^{v} \operatorname{tr} \left(\mathbf{A}^{(p)} \mathbf{Z} \mathbf{Z}^{\top} \mathbf{A}^{(p)\top}\right)-2 \operatorname{tr}\left(\mathbf{Z} \mathbf{X}^{(p)\top} \mathbf{A}^{(p)}\right) \\ 
 & \quad +\lambda \operatorname{tr}\left(\mathbf{A}^{(p)} \operatorname{diag}(\mathbf{Z} \mathbf{1}) \mathbf{A}^{(p)\top}\right).
\end{aligned}
\end{equation}

Considering the optimization of each $\mathbf{A}_v$ is independent of the corresponding view. Therefore, we can optimize Eq. (\ref{opt a}) by solving the problems as follows,

\begin{equation}
\min_{\mathbf{A}^{(p)}} 2 \mathbf{A}^{(p)} \mathbf{Z} \mathbf{Z}^{\top}+2 \lambda \mathbf{A}^{(p)} \operatorname{diag}(\mathbf{Z} \mathbf{1})- 2 \mathbf{X}^{(p)} \mathbf{Z}^{\top}.
\end{equation}

Intuitively, $\mathbf{A}^{(p)}$ can be calculated by:

\begin{equation}
\mathbf{A}^{(p)}= \mathbf{X}^{(p)} \mathbf{Z}^{\top}\left(\mathbf{Z} \mathbf{Z}^{\top}+\lambda \operatorname{diag}(\mathbf{Z} \mathbf{1})\right)^{-1}.
\end{equation}

\subsubsection{Optimization of Consistent Anchor Graph \texorpdfstring{$\mathbf{Z}$}{}}
When $\{\mathbf{A}^{(p)}\}_{p=1}^v$ is fixed, the optimization for $\mathbf{Z}$ can be written as follows:
\begin{equation}
\label{opt z}
\begin{aligned}
\min_{\mathbf{Z}}  &\quad \sum_{p=1}^{v}\operatorname{tr}\left( \mathbf{Z}^{\top} \mathbf{A}^{(p)\top} \mathbf{A}^{(p)} \mathbf{Z}\right)+\lambda \sum_{p=1}^{v}\operatorname{tr}(\mathbf{M}^{(p)} \mathbf{Z})\\
& -2 \sum_{p=1}^{v}\operatorname{tr}\left(\mathbf{X}^{(p)\top} \mathbf{A}^{(p)} \mathbf{Z}\right) + \mu \operatorname{tr}\left(\mathbf{Z}^{\top} \mathbf{Z}\right), \\
& s.t.\mathbf{Z} \geq 0, \mathbf{Z}^{\top} \mathbf{1}=\mathbf{1},
\end{aligned}
\end{equation}
where $\mathbf{M}^{(p)}_{i,:}=\left[\mathbf{q}_1^{(p)} , \cdots , \mathbf{q}_m^{(p)}\right]$
, $\mathbf{q}_j^{(p)} = \mathbf{A}^{(p)\top}_{:,j} \mathbf{A}^{(p)}_{:,j}.$

As mentioned earlier, the above optimization process of $\mathbf{Z}$ can be easily expressed as the below  Quadratic Programming (QP) problem,

\begin{equation}
\begin{aligned}
&\min \frac{1}{2} \mathbf{Z}_{:, i}^{\top} \mathbf{Q} \mathbf{Z}_{:, i}+f_i^{\top} \mathbf{Z}_{:, i}, \\
&\text { s.t. } \mathbf{Z}_{:, i}^{\top} \mathbf{1}=1, \mathbf{Z}_{:, i} \geq 0,
\end{aligned}
\end{equation}
where $\mathbf{Q}= \sum_{p=1}^{v}\mathbf{A}^{(p)\top} \mathbf{A}^{(p)} + \mu \mathbf{I}$,$f_i=- \sum_{p=1}^{v}\left({\mathbf{X}_{:, i}^{(p)}}^{\top}\mathbf{A}^{(p)}+ \frac{\lambda}{2} \mathbf{M}_{i,:}^{(p)} \right)$.

Because each column of $\mathbf{Z}$ is a $m$-dimensional vector, the time complexity of this sub-problem is $\mathcal{O}(nm^3)$.
The entire optimization process for solving Eq.(\ref{my formula}) is summarized in Algorithm 1.

\begin{algorithm}[t]
	\renewcommand{\algorithmicrequire}{\textbf{Input:}}
	\renewcommand{\algorithmicensure}{\textbf{Output:}}
	\caption{Efficient Multi-View Graph Clustering with Local and Global Structure Preservation (EMVGC-LG)}
	\label{my alg}
	\begin{algorithmic}[1]
		\REQUIRE multi-views dataset $\{\mathbf{X}^{(p)}\}_{p=1}^{v}$, the number of cluster $k$.
		
		\STATE Initialize  $\{\mathbf{A}^{(p)}\}_{p=1}^{v}$. 
         \REPEAT
		
		\STATE Obtain $\{\mathbf{A}^{(p)}\}_{p=1}^{v}$ with Eq. (\ref{opt a}).
		\STATE Obtain $\mathbf{Z}$ with Eq. (\ref{opt z}).
  \UNTIL{converged}
\STATE Obtain $\mathbf{H}$ by performing SVD on $\mathbf{Z}$
		%\STATE \textbf{return} $\mathbf{H}^*_m$
  \ENSURE Perform k-means on $\mathbf{H}$ to obtain discrete labels.
	\end{algorithmic}  
	\label{algo_whole}
\end{algorithm}

\subsection{Discussions}
\subsubsection{Convergence}
As the iterations proceed, two variables of the above optimization procedure will be separately addressed. Since each subproblem has reached the global optimum, the value of the EMVGC-LG function will monotonically decrease and finally reach convergence. Moreover, since the lower bound of the objective function can be easily proved to be $- \sum_{p=1}^v  \operatorname{tr}( \mathbf{X}^{(p)\top} \mathbf{X}^{(p)})$ (by Proposition \ref{pro}), our proposed EMVGC-LG is proven to converge to the local optimum.

\subsubsection{Space Complexity}
In this paper, the primary memory overhead of our approach is the matrix $\mathbf{Z} \in \mathbb{R}^{m \times n}$ and $\mathbf{A}^{(p)} \in \mathbb{R}^{d_p \times m}$. As a result, the space complexity of EMVGC-LG is $m\left(n+d\right)$, where $d = \sum_{p=1}^v d_p$. $m \ll n$ and $d \ll n$ are within our settings. Therefore, the space complexity of EMVGC-LG is $\mathcal{O}(n)$. 

\subsubsection{Time Complexity}
The time complexity of EMVGC-LG is composed of two optimization steps, as previously mentioned. The time complexity of updating $\{\mathbf{A}^{(p)}\}_{p=1}^v$ is $\mathcal{O}\left((nmd + m^3)v\right)$.  When analytically obtaining $\mathbf{Z}$, it costs $\mathcal{O}(nm^3)$ for all columns. Therefore, the total time cost of the optimization process is $\mathcal{O}\left(n(mdv + m^3) + m^3v\right)$. Consequently, the computational complexity of EMVGC-LG is $\mathcal{O}(n)$,  which is linearly related to the number of data.

\section{Experiment}
In this section, we perform numerous experiments to assess our proposed EMVGC-LG. Concretely, we discuss the clustering quality on synthetic and real datasets, the evolution of the objective values, the running time, the sensitivity of the parameters, and the ablation study. Our code is accessible on the \url{https://github.com/wy1019/EMVGC-LG}.

\subsection{Synthetic Datasets}

    In order to visualize the different influences of local and global structure, we performed experiments on a synthetic two-view two-dimensional dataset containing 500 samples extracted from five clusters produced by a Gaussian function. From Figure \ref{our:synthetic}, we can observe that incorporating local and global information can effectively enhance the clustering performance. Compared with a singular structure, our strategy improves the performance by 11.85\% and 26.08\%, respectively, and yields a clearer partition of clusters.

\begin{table}[t]
\caption{Benchmark datasets}
\renewcommand\arraystretch{1}
\tabcolsep=0.2cm
	\centering
    \scalebox{1.05}{
\begin{tabular}{cccc}
\toprule
Datasets      & Samples & Views & Clusters \\ \hline
Reuters12     & 1200    & 5     & 6        \\
Flower17      & 1360    & 7     & 17       \\
Mfeat         & 2000    & 2     & 10       \\
VGGFace50     & 16936   & 4     & 50       \\
Caltech256    & 30607   & 4     & 257      \\
YouTubeFace10 & 38654   & 4     & 10       \\
Cifar10       & 60000   & 4     & 9        \\
Cifar100      & 60000   & 4     & 99       \\
YouTubeFace20 & 63896   & 4     & 20       \\
YouTubeFace50 & 126054  & 4     & 50       \\
\bottomrule
\end{tabular}}
\label{benchmark_data}
\end{table}

\subsection{Real-world Datasets}
Ten extensively available datasets were used to assess the validity of the proposed algorithms. The information of the datasets are as follows, including Reuters12\footnote{\url{https://archive.ics.uci.edu/ml/datasets/reuters-21578+text+categorization+collection}}, Flower17\footnote{\url{https://www.robots.ox.ac.uk/ vgg/data/flowers/17/}}, 
Mfeat\footnote{\url{http://www.svcl.ucsd.edu/projects/crossmodal/}}, 
VGGFace50\footnote{\url{https://www.robots.ox.ac.uk/ vgg/data/vgg_face/}}, 
Caltech256\footnote{\url{https://authors.library.caltech.edu/7694/}}, 
YouTubeFace10\footref{fn:1}, 
Cifar10\footref{fn:2}, Cifar100\footnote{\url{https://www.cs.toronto.edu/ kriz/cifar.html}\label{fn:2}}, YouTubeFace20\footref{fn:1}, 
and YouTubeFace50\footnote{\url{https://www.cs.tau.ac.il/ wolf/ytfaces/}\label{fn:1}}. The detailed information is listed in Table \ref{benchmark_data}. Specifically, 
Reuters12 is a subset of Reuters, which contains 1200 documents described in five languages, including English, France, German, Italian, and Spanish.
Flower17  is a flower dataset with 80 images for each class. 
Mfeat contains 2000 images of handwritten numbers from 0 to 9. Each sample is expressed by six different feature sets,  i.e., 216-dimensional FAC, 76-dimensional FOU,  64-dimensional KAR, 6 MORs, 240-dimensional Pix, and 47-dimensional ZER. 
VGGFace50 is derived from VGGFace. 
Caltech256  contains 30607 images spanning 257 object categories. 
Features were extracted from 60,000 tiny color images of 10 categories and 100 categories in four views of Cifar10 and Cifar100 by ResNet18, ResNet50, and DenseNet121.
YoutubeFace10, YoutubeFace20, and YouTubeFace50 are face video datasets withdrawn from YouTube. 

\begin{table*}[t]
\caption{Empirical evaluation and comparison of EMVGC-LG with ten baseline methods on ten benchmark datasets}
\label{results}
\selectfont 
\centering
\renewcommand\arraystretch{1}
\tabcolsep=0.05cm
\scalebox{1}{
\begin{tabular}{cccccccccccc}
\toprule
Dataset       & RMKM                              & AMGL       & FMR                               & PMSC       & BMVC                              & LMVSC                             & SMVSC                             & SFMC       & FMCNOF     & FPMVS-CAG  & Proposed                          \\ \hline
\multicolumn{12}{c}{ACC (\%)}                                                                                                                                                                                                                                                                          \\ \hline
Reuters12     & 44.00±0.00                        & 18.12±0.42 & 52.35±1.90                        & 27.88±0.75 & 50.33±0.00                        & 47.92±2.98                        & {\color[HTML]{4F81BD} 55.69±3.06} & 17.08±0.00 & 28.58±0.00 & 45.82±3.29 & {\color[HTML]{FF0000} 60.74±2.01} \\
Flower17      & 23.24±0.00                        & 9.70±1.53  & 33.43±1.75                        & 20.82±0.74 & 26.99±0.00                        & {\color[HTML]{4F81BD} 37.12±1.86} & 27.13±0.84                        & 7.57±0.00  & 17.43±0.00 & 25.99±1.83 & {\color[HTML]{FF0000} 53.28±2.51} \\
Mfeat         & 80.80±0.00                        & 71.42±5.42 & 64.53±2.28                        & 66.41±4.40 & 65.80±0.00                        & {\color[HTML]{4F81BD} 82.86±6.95} & 65.57±3.99                        & 56.90±0.00 & 56.95±0.00 & 65.11±4.18 & {\color[HTML]{FF0000} 88.53±4.87} \\
VGGFace50     & 8.23±0.00                         & 2.95±0.35  & O/M                               & O/M        & 10.30±0.00                        & 10.56±0.26                        & {\color[HTML]{4F81BD} 13.36±0.60} & 3.64±0.00  & 5.51±0.00  & 12.06±0.36 & {\color[HTML]{FF0000} 15.29±0.74} \\
Caltech256    & 9.87±0.00                         & O/M        & O/M                               & O/M        & 8.63±0.00                         & 9.57±0.17                         & {\color[HTML]{4F81BD} 10.54±0.15} & O/M        & 2.70±0.00  & 8.78±0.07  & {\color[HTML]{FF0000} 11.70±0.36} \\
YouTubeFace10 & {\color[HTML]{4F81BD} 74.88±0.00} & O/M        & O/M                               & O/M        & 58.58±0.00                        & 74.48±5.30                        & 72.93±3.96                        & 55.80±0.00 & 43.42±0.00 & 67.09±5.70 & {\color[HTML]{FF0000} 79.55±4.57} \\
CIFAR10       & O/M                               & O/M        & O/M                               & O/M        & 27.81±0.00                        & 29.02±0.81                        & {\color[HTML]{4F81BD} 29.11±1.35} & 10.02±0.00 & 20.53±0.00 & 26.89±0.71 & {\color[HTML]{FF0000} 31.14±1.07} \\
CIFAR100      & O/M                               & O/M        & O/M                               & O/M        & 8.32±0.00                         & {\color[HTML]{4F81BD} 9.53±0.15}  & 8.34±0.17                         & 1.18±0.00  & 3.66±0.00  & 7.29±0.11  & {\color[HTML]{FF0000} 10.96±0.35} \\
YouTubeFace20 & O/M                               & O/M        & O/M                               & O/M        & 57.39±0.00                        & {\color[HTML]{4F81BD} 67.26±3.53} & 67.13±4.20                        & O/M        & 38.61±0.00 & 63.08±3.79 & {\color[HTML]{FF0000} 72.79±2.73} \\
YouTubeFace50 & O/M                               & O/M        & O/M                               & O/M        & 66.00±0.00                        & {\color[HTML]{4F81BD} 68.32±2.45} & O/M                               & O/M        & 21.66±0.00 & 64.24±2.97 & {\color[HTML]{FF0000} 70.52±2.62} \\ \hline
\multicolumn{12}{c}{NMI (\%)}                                                                                                                                                                                                                                                                          \\ \hline
Reuters12     & 26.83±0.00                        & 3.46±0.76  & 31.43±0.84                        & 14.76±0.58 & 27.25±0.00                        & 27.86±1.49                        & {\color[HTML]{4F81BD} 32.46±1.72} & 12.61±0.00 & 7.15±0.00  & 24.59±3.20 & {\color[HTML]{FF0000} 35.80±1.03} \\
Flower17      & 22.07±0.00                        & 10.25±4.01 & 30.65±0.91                        & 19.13±0.48 & 25.62±0.00                        & {\color[HTML]{4F81BD} 35.37±1.10} & 25.78±0.76                        & 7.87±0.00  & 14.68±0.00 & 25.81±1.59 & {\color[HTML]{FF0000} 51.77±1.72} \\
Mfeat         & 82.28±0.00                        & 77.12±2.19 & 66.21±0.73                        & 63.29±1.63 & 59.39±0.00                        & {\color[HTML]{4F81BD} 82.46±2.79} & 57.99±2.11                        & 68.15±0.00 & 55.47±0.00 & 57.77±2.73 & {\color[HTML]{FF0000} 82.90±2.78} \\
VGGFace50     & 9.66±0.00                         & 2.04±0.50  & O/M                               & O/M        & 13.48±0.00                        & 12.64±0.28                        & {\color[HTML]{4F81BD} 16.21±0.49} & 1.63±0.00  & 4.74±0.00  & 14.74±0.55 & {\color[HTML]{FF0000} 18.71±0.76} \\
Caltech256    & 31.01±0.00                        & O/M        & O/M                               & O/M        & 31.83±0.00                        & {\color[HTML]{4F81BD} 31.96±0.11} & 28.27±0.24                        & O/M        & 1.60±0.00   & 22.97±0.21 & {\color[HTML]{FF0000} 34.78±0.17} \\
YouTubeFace10 & {\color[HTML]{4F81BD} 78.83±0.00} & O/M        & O/M                               & O/M        & 54.66±0.00                        & 77.74±2.03                        & 78.57±2.80                        & 77.46±0.00 & 39.15±0.00 & 76.11±3.06 & {\color[HTML]{FF0000} 82.21±2.97} \\
CIFAR10       & O/M                               & O/M        & O/M                               & O/M        & {\color[HTML]{4F81BD} 17.90±0.00} & 17.84±0.53                        & 16.00±0.99                        & 0.16±0.00  & 10.33±0.00 & 15.45±0.99 & {\color[HTML]{FF0000} 18.22±0.66} \\
CIFAR100      & O/M                               & O/M        & O/M                               & O/M        & 15.05±0.00                        & {\color[HTML]{4F81BD} 15.40±0.18} & 14.40±0.20                        & 0.53±0.00  & 7.04±0.00  & 13.62±0.16 & {\color[HTML]{FF0000} 18.42±0.33} \\
YouTubeFace20 & O/M                               & O/M        & O/M                               & O/M        & 70.65±0.00                        &  76.78±1.34 & {\color[HTML]{4F81BD}78.36±2.39}                        & O/M        & 45.45±0.00 & 74.30±1.95 & {\color[HTML]{FF0000} 80.57±1.60} \\
YouTubeFace50 & O/M                               & O/M        & O/M                               & O/M        & 81.90±0.00                        & {\color[HTML]{4F81BD} 82.43±0.78} & O/M                               & O/M        & 43.03±0.00 & 82.08±1.07 & {\color[HTML]{FF0000} 84.17±0.83} \\ \hline
\multicolumn{12}{c}{Fscore (\%)}                                                                                                                                                                                                                                                                       \\ \hline
Reuters12     & 32.18±0.00                        & 28.36±0.01 &  38.62±0.43 & 27.25±0.41 & 35.49±0.00                        & 34.80±1.16                        & {\color[HTML]{4F81BD}39.12±1.10}                        & 27.71±0.00 & 24.11±0.00 & 34.14±1.23 & {\color[HTML]{FF0000} 43.01±1.29} \\
Flower17      & 14.35±0.00                        & 11.49±0.34 & 20.09±0.81                        & 12.33±0.36 & 16.61±0.00                        & {\color[HTML]{4F81BD} 23.99±0.95} & 17.53±0.21                        & 10.94±0.00 & 13.93±0.00 & 17.29±0.39 & {\color[HTML]{FF0000} 38.41±1.74} \\
Mfeat         & 76.64±0.00                        & 65.74±7.46 & 58.89±1.02                        & 55.68±2.20 & 50.87±0.00                        & {\color[HTML]{4F81BD} 78.11±5.34} & 53.09±2.65                        & 45.62±0.00 & 45.52±0.00 & 52.50±3.10 & {\color[HTML]{FF0000} 80.93±4.85} \\
VGGFace50     & 3.69±0.00                         & 3.91±0.02  & O/M                               & O/M        & 5.10±0.00                         & 5.09±0.15                         & {\color[HTML]{4F81BD} 6.35±0.18}  & 4.16±0.00  & 4.36±0.00  & 6.10±0.07  & {\color[HTML]{FF0000} 7.93±0.41}  \\
Caltech256    & {\color[HTML]{4F81BD} 7.28±0.00}  & O/M        & O/M                               & O/M        & 6.25±0.00                         & 6.00±0.28                         & 5.37±0.20                         & O/M        & 1.17±0.00  & 3.22±0.02  & {\color[HTML]{FF0000} 8.70±0.69}  \\
YouTubeFace10 & 66.93±0.00                        & O/M        & O/M                               & O/M        & 52.53±0.00                        & {\color[HTML]{4F81BD} 68.93±3.19} & 68.34±5.78                        & 61.25±0.00 & 32.88±0.00 & 66.10±5.06 & {\color[HTML]{FF0000} 75.38±4.44} \\
CIFAR10       & O/M                               & O/M        & O/M                               & O/M        & {\color[HTML]{FF0000} 21.64±0.00} & 20.34±0.43                        & 19.70±0.55                        & 18.18±0.00 & 18.48±0.00 & 20.58±0.34 & {\color[HTML]{4F81BD} 20.85±0.54} \\
CIFAR100      & O/M                               & O/M        & O/M                               & O/M        & {\color[HTML]{4F81BD} 4.37±0.00}  & 3.69±0.05                         & {\color[HTML]{FF0000}4.47±0.06}                         & 1.98±0.00  & 2.56±0.00  & 3.78±0.01  & {\color[HTML]{4F81BD} 4.36±0.15}  \\
YouTubeFace20 & O/M                               & O/M        & O/M                               & O/M        & 49.04±0.00                        & {\color[HTML]{4F81BD} 62.43±2.91} & 61.68±5.99                        & O/M        & 25.84±0.00 & 57.81±4.00 & {\color[HTML]{FF0000} 68.22±3.48} \\
YouTubeFace50 & O/M                               & O/M        & O/M                               & O/M        & 57.09±0.00                        & {\color[HTML]{4F81BD} 62.49±2.45} & O/M                               & O/M        & 15.67±0.00 & 56.89±3.18 & {\color[HTML]{FF0000} 63.34±1.78} \\
\bottomrule
\end{tabular}}
\end{table*}

\begin{figure*}[t]
\begin{center}
{
\centering
\subfloat[Local (ACC = 73.08\%)]{\includegraphics[height = 130 pt, width = 140 pt]{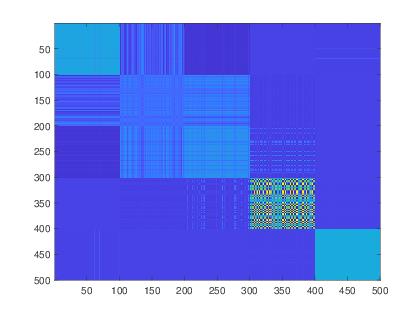}}
\subfloat[Global (ACC = 88.15\%)]{\includegraphics[height = 130 pt, width = 140 pt]{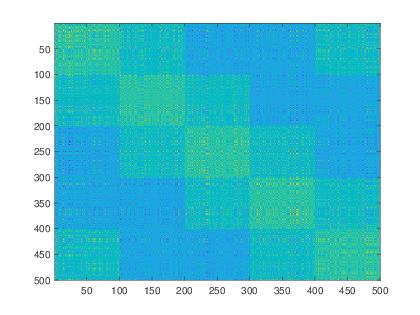}}
\subfloat[Ours (ACC = 100\%)]{\includegraphics[height = 130 pt, width = 140 pt]{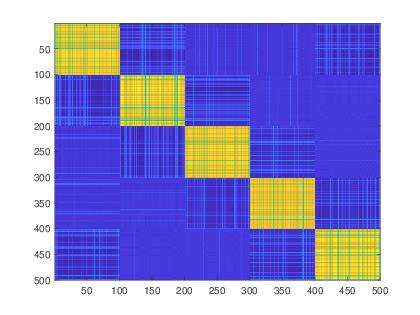}}
\caption{Visualization of learned graph on synthetic datasets.}
\label{our:synthetic}
}
\end{center}
\end{figure*}

\begin{figure*}[t]
\centering
\includegraphics[width=1\linewidth]{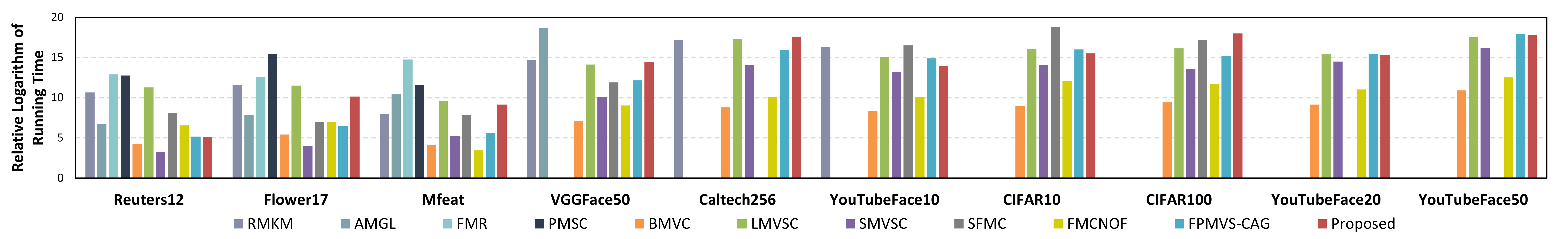}
\caption{Time comparison of different MVC Methods on ten datasets}
\label{time}
\end{figure*}

\begin{figure*}[t]
\begin{center}{
    \centering  
    \subfloat{
            \includegraphics[width=1\linewidth]{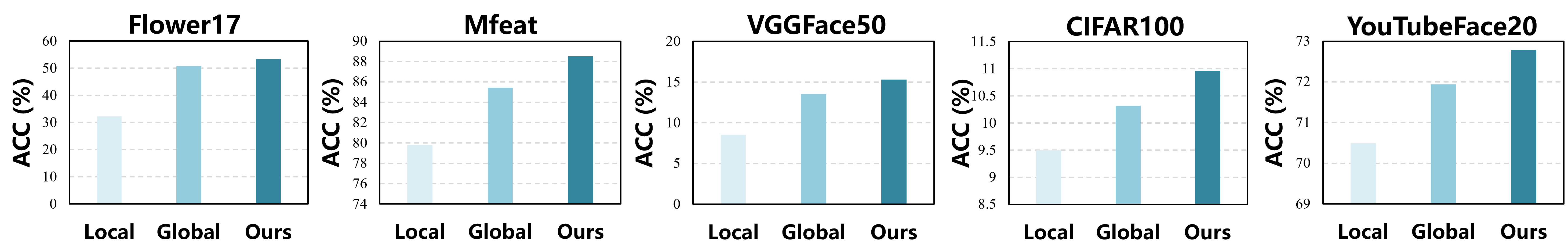}
    }
    }
    \end{center}
\caption{The ablation study of our local and global structure combination strategy on five benchmark datasets.}
\label{abla_lg}
\end{figure*}

\begin{figure*}[t]
\begin{center}{
    \centering  
    \subfloat{
            \includegraphics[width=1\linewidth]{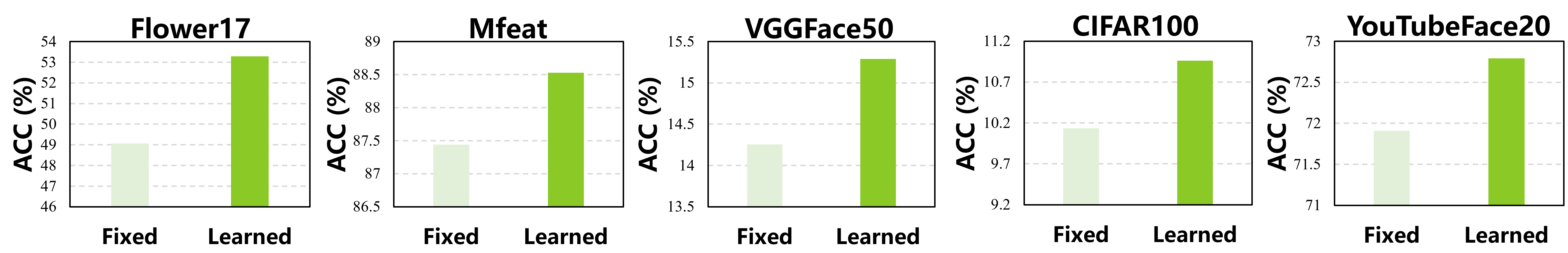}
    }
    }
    \end{center}
\caption{The ablation study of our anchor learning strategy on five benchmark datasets. "Fixed" indicates without our anchor learning strategy.}
\label{abla_anchor}
\end{figure*}
\subsection{Compared Methods}
Along with our proposed EMVGC-LG,  we run ten state-of-the-art multi-view graph clustering methods for comparison, including  Multi-view k-means Clustering on Big Data (RMKM) \cite{cai2013multi},  Parameter-free Auto-weighted Multiple Graph Learning (AMGL) \cite{nie2016parameter}, Flexible Multi-View Representation Learning for Subspace Clustering (FMR) \cite{li2019flexible},
 Partition Level Multi-view Subspace Clustering (PMSC) \cite{kang2020partition},
 Binary Multi-View Clustering (BMVC)  \cite{zhang2018binary},
Large-scale Multi-view Subspace
Clustering in Linear Time (LMVSC)\cite{kang2020large},
Scalable Multiview Subspace Clustering with Unified Anchors (SMVSC) \cite{sun2021scalable},
Multi-view clustering: a Scalable and Parameter-free Bipartite Graph Fusion Method (SFMC) \cite{li2020multi},
Fast Multiview Clustering via Nonnegative and Orthogonal Factorization
 (FMCNOF)\cite{yang2020fast2},
and Fast Parameter-free Multiview Subspace Clustering with Consensus Anchor Guidance (FPMVS-CAG)\cite{wang2022highly}.

For all the aforementioned algorithms, we set their parameters
as their recommended range. In the proposed method, we adjusted $\lambda$ to $[10^{-3},10^{-2},10^{-1},1]$, $\mu$ to
$[0, 10^{-4},1,10^4]$, and the anchor numbers to [k, 2k, 5k], with a grid search scheme. In addition, we repeated each experiment 10 times to calculate the mean performance and standard deviation. To assess the clustering performance, we use three widely used metrics, consisting of Accuracy (ACC), Normalised Mutual Information (NMI), and Fscore. All experiments were conducted on a desktop computer with Intel core i9-10900X CPU and 64G RAM, MATLAB 2020b (64-bit).

\subsection{Experimental Results}

Table \ref{results} indicates the clustering performance on ten datasets. The best results are marked in red, and the second-best results are marked in blue. "O/M" represents the unavailable results due to time-out or out-of-memory errors. 
According to the results, we have the following observation:

\begin{enumerate}[(1)]
    \item Compared with existing MVC methods, our proposed algorithm achieves the best or second-best performance on ten datasets. In comparison with the second-best ones, our EMVGC-LG acquires 5.05\%, 16.16\%, 5.67\%, 1.93\%, 1.16\%, 4.67\%, 2.03\%, 1.43\%, 5.53\% and 2.2\% in terms of ACC, which demonstrates local and global structure fusion strategy. In other metrics, EMVGC-LG also achieved desirable performance.
    \item Comparison with classical MVGC methods, i.e., RMKM, AMGL, FMR, and PMSC, encounter scalability problems on large-scale datasets due to the huge matrix computation and memory generated by the full graph construction. Our EMVGC-LG outperforms them by 8.39\%, 19.85\%, 7.73\%, 7.06\%, 1.87\%, and 4.67\% of ACC on six datasets, showing the superiority of our anchor-based algorithm.
    \item Comparison with existing AMVGC methods, i.e., BMVC, LMVSC, SMVSC, SFMC, FMCNOF, and FPMVS-CAG, our EMVGC-LG still acquires comparable or better performance. In particular, the LMVSC method shows better results than other methods, which demonstrates its superiority in large-scale scenarios. In terms of ACC, our EMVGC-LG achieves better performance than LMVSC by large margins of 12.82\%, 16.16\%, 5.67\%, 4.73\%, 2.13\%, 5.07\%, 2.12\%, 1.43\%, 5.53\%, and 2.20\%, illustrating our effectiveness.
\end{enumerate}

\begin{figure}[t]
\centering    
\begin{center}{
    \centering
     \begin{minipage}[t]{1\linewidth}
        \centering
        \subfloat{
            \includegraphics[width=1\linewidth]{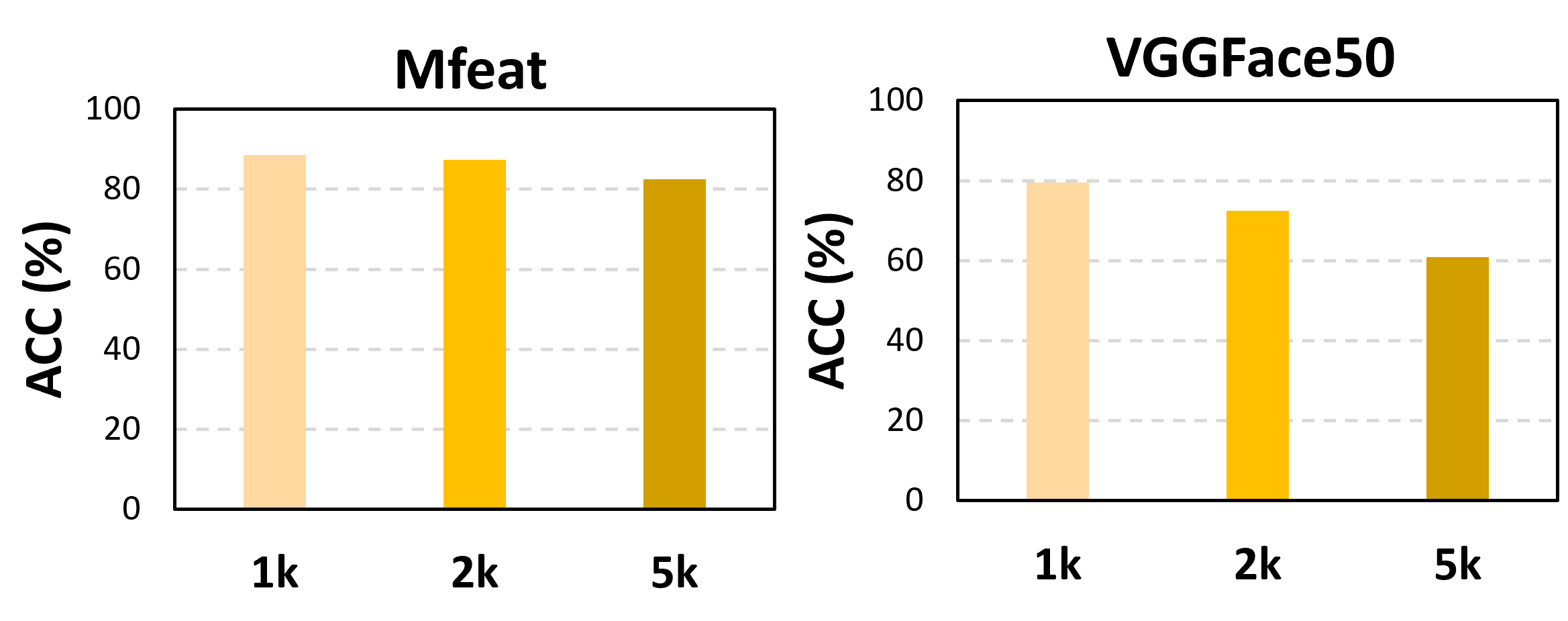}
        }
    \end{minipage}}
    \end{center}     
\caption{Sensitivity analysis of anchor number m of our method on two benchmark datasets.}     
\label{ablation}     
\end{figure}

\subsection{Running Time Comparison}
To validate the computational efficiency of the proposed EMVGC-LG, we plot the average running time of each algorithm on ten benchmark datasets in Figure \ref{time}. The results of some compared algorithms on large-scale datasets are not reported due to memory overflow errors caused by their excessive time and space complexity. As shown in the Figure \ref{time}, we can observe that
\begin{enumerate} [(1)]
    \item Compared to full graph-based clustering methods, the proposed EMVGC-LG significantly reduces run time through the construction of anchor graphs.
    \item Compared to the anchor-based MVC approach, i.e., SMVSC and FPMVS-CAG, the proposed EMVGC-LG requires more time consumption, mainly due to our local and global structure preservation strategy. Generally, the extra time spent is worthwhile since our proposed EMVGC-LG demonstrates its superiority in most datasets.
\end{enumerate}

\subsection{Ablation Study}
% \subsection{Ablation Study on Structural Alignment Strategy}

\subsubsection{Local and Global Structure Combination Strategy}
The local and global structure combination strategy is the main contribution of this paper. To further demonstrate the effectiveness of this strategy, we present the experimental results of the ablation study in Figure \ref{abla_lg}, where "Local" and "Global" indicates only using the local or global structure, respectively. In our experimental setting, we optimize the Eq.\eqref{local} and Eq.\eqref{global} in the optimization process to obtain the final clustering result. In terms of ACC, the proposed structure combination strategy improves the algorithm performance on the Flower17, Mfeat, VGGFace50, Cifar100, and YouTubeFace20 datasets by \textbf{21.03\%,8.71\%, 6.74\%, 1.46\%}, and \textbf{2.30}\% compared to the simple local structure respectively, which demonstrates the effectiveness of our strategy. 

\begin{figure}[t]
\begin{center}{
    \centering
        \subfloat{
            \includegraphics[width=1\linewidth]{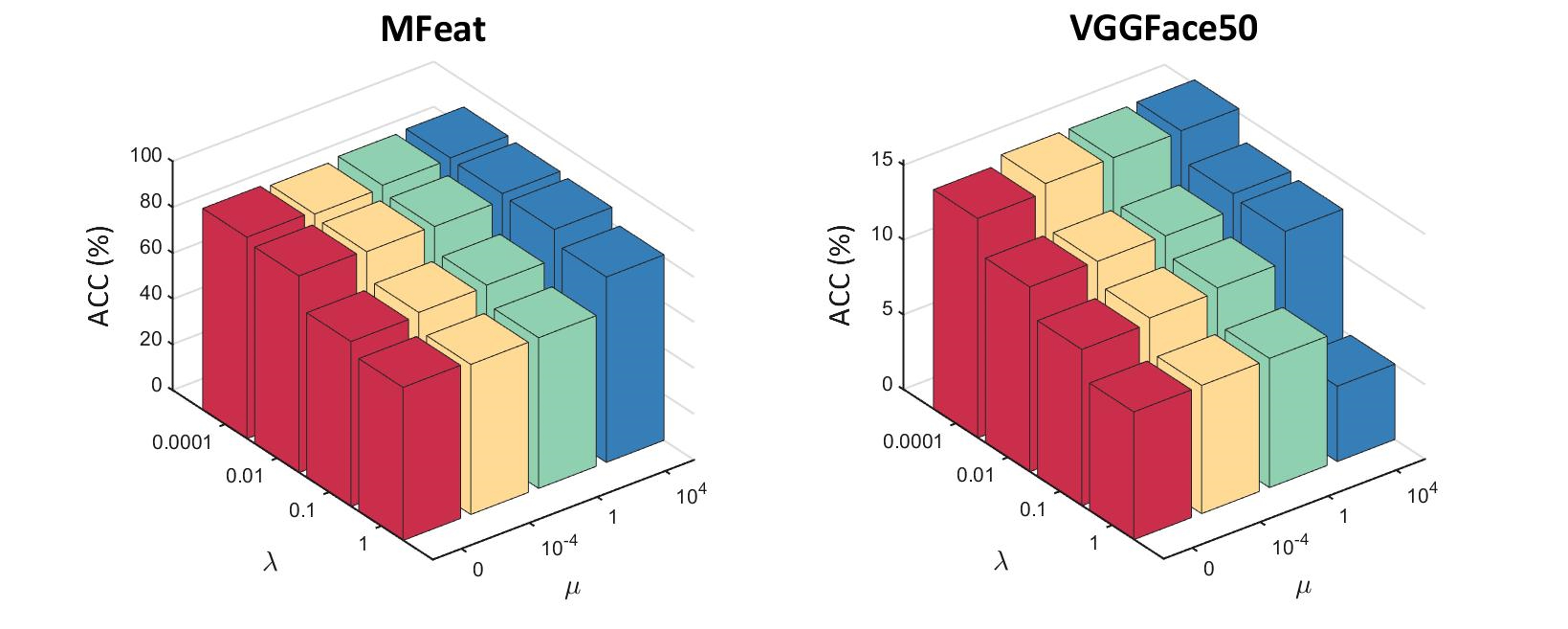}}
    }
    \end{center}
\caption{Sensitivity analysis of $\lambda$ and $\mu$ of our method on two benchmark datasets.}
\label{ablation2}
\end{figure}

\subsubsection{Anchor Learning Strategy}
We conducted ablation experiments with the proposed anchor learning strategy, as shown in Figure \ref{abla_anchor}. "Fixed" indicates initializing anchors by k-means without updating during the optimization process. Compared to the above methods, our approach significantly improves the clustering performance and avoids the high time expenditure of k-means, demonstrating the effectiveness of the anchor learning strategy.

\subsection{Convergence and Sensitivity}
We conducted several experiments to demonstrate the convergence of the proposed algorithm. As shown in Figure \ref{obection}, the objective value of our algorithm is monotonically decreasing in each iteration. These results clearly verify the convergence of our proposed algorithm. 
\begin{figure}
\begin{center}{
    \centering
     \begin{minipage}[t]{1\linewidth}
        \centering
        \subfloat{
            \includegraphics[width=1\linewidth]{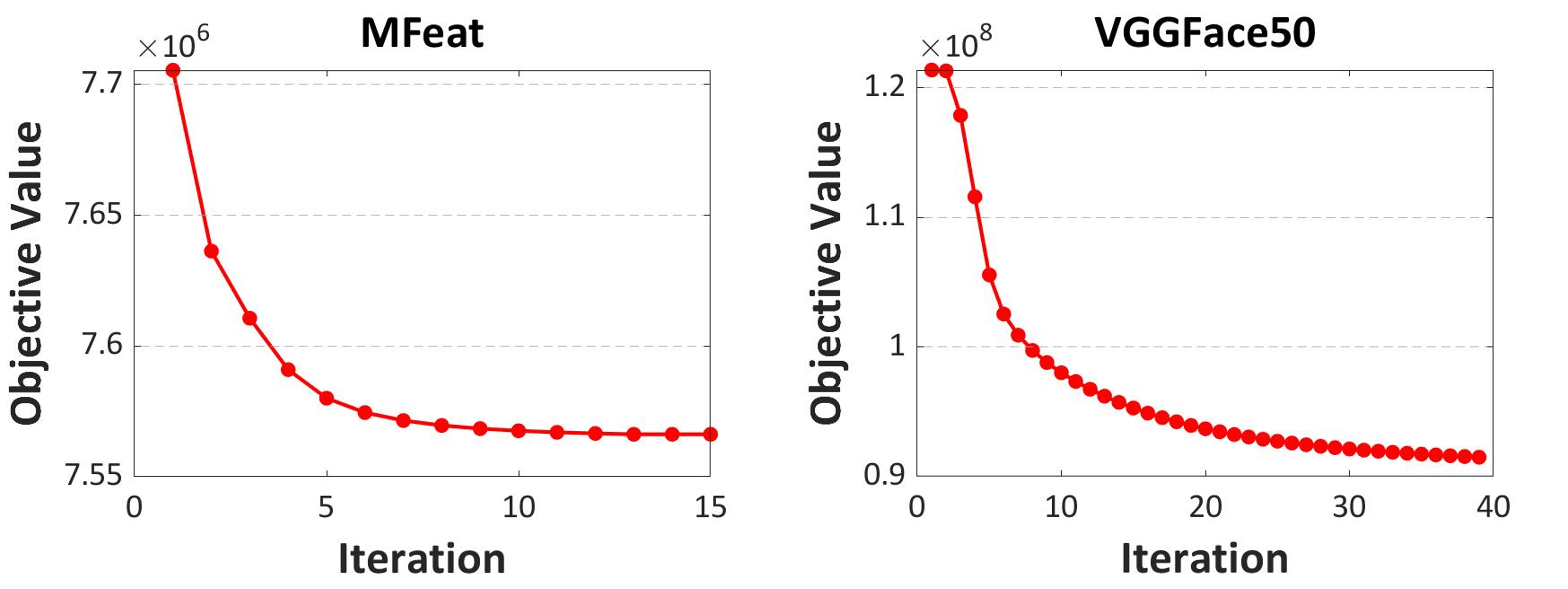}
        }
    \end{minipage}}
    \end{center}
\caption{Objective values of the proposed method on two benchmark datasets.}
\label{obection}
\end{figure}
To investigate the sensitivity of our algorithm to the number of anchors m, we investigated how our performance shifts for different numbers of anchors. As shown in Figure \ref{ablation}, the number of anchors has some effect on the performance of our algorithm and basically achieves optimal performance at m = k. Moreover, two hyperparameters, $\lambda$, and $\mu$, are used in our method, $\lambda$ is the local and global structure balanced parameter, and $\mu$ is the coefficient of the sparsity regularization term. As is shown in Figure \ref{ablation2}, we conducted comparative experiments on two benchmark datasets to illustrate the impact of these two parameters on performance. In terms of VGGFace50, our method performs better when $\lambda$ is less than 0.001, and $\mu$ is larger than 1. When the value of $\lambda$ is larger than 0.01, EMVGC-LG works well on the Mfeat dataset, and $\mu$ has little effect on it. Thus, with fixed $\lambda$, the variation of $\mu$ has a smaller effect on the final performance on most datasets, while ACC with the same $\mu$ will be affected by $\lambda$.

\section{Conclusion}
In this paper, we propose a novel anchor-based multi-view graph clustering framework termed Efficient Multi-View Graph Clustering with Local and Global Structure Preservation (EMVGC-LG). Specifically, EMVGC-LG considers preserving the local and global structures in a unified framework, which provides comprehensive information for clustering. We theoretically prove that the proposed paradigm with the global structure can well approximate the local information. Besides, anchor construction and graph learning are jointly optimized in our unified framework to enhance the clustering quality. In addition, EMVGC-LG inherits the linear complexity of existing AMVGC methods respecting the sample number, which is time-economical and scales well with the data size. Extensive experiments demonstrate the effectiveness and efficiency of our proposed method. In the future, we will explore the relationship between local and global structure in more detail, e.g., under what circumstances is local structure preferable to global structure?

\section{ACKNOWLEDGMENTS}
This work was supported by the National Key R\&D Program of China (no. 2020AAA0107100) and the National Natural Science Foundation of China (project no. 62325604, 62276271).

% \clearpage
\bibliographystyle{ACM-Reference-Format}
\balance
% \newpage
\bibliography{sample-base}
\end{document}